\documentclass{article}

\def \TRtitle{Kernel-based Distance Metric Learning \\ in the Output Space}
\def \TRauthors{Cong Li, Michael Georgiopoulos and Georgios C. Anagnostopoulos}
\def \TRemails{congli@eecs.ucf.edu, michaelg@ucf.edu and georgio@fit.edu}
\def \TRaffiliations{CL, MG: EE \& CS Dept., University of Central Florida; GCA: ECE Dept., Florida Institute of Technology}
\def \TRkeywords{Distance Metric Learning, Kernel Methods, Reproducing Kernel Hilbert Space for Vector-valued Functions}

\usepackage{./sty/options}
\usepackage{./sty/packages}
\usepackage{./sty/macros}

\begin{document}

\maketitle

\ifMakeReviewDraft
	\linenumbers
\fi

\begin{abstract}
In this paper we present two related, kernel-based \ac{DML} methods. Their respective models non-linearly map data from their original space to an output space, and subsequent distance measurements are performed in the output space via a Mahalanobis metric. The dimensionality of the output space can be directly controlled to facilitate the learning of a low-rank metric. Both methods allow for simultaneous inference of the associated metric and the mapping to the output space, which can be used to visualize the data, when the output space is 2- or 3-dimensional. Experimental results for a collection of classification tasks illustrate the advantages of the proposed methods over other traditional and kernel-based \ac{DML} approaches.
\end{abstract}



\section{Introduction}
\label{sec:Introduction}

\acresetall

\ac{DML} has become an active research area due to the fact that many machine learning models and algorithms depend on metric calculations. Considering plain Euclidean distances between samples may not be a suitable approach for some practical problems, \eg, for \ac{KNN} classification, where a metric other than the Euclidean may yield higher recognition rates. Hence, it may be important to learn an appropriate metric for the learning problem at hand. \ac{DML} aims to address this problem, \ie, to infer a parameterized metric from the available training data that maximizes the performance of a model. 

Most of past \ac{DML} research focuses specifically on learning a weighted Euclidean metric, also known as the Mahalanobis distance (\eg\ see \cite{Maesschalck2000}), or generalizations of it, where the weights are inferred from the data. For elements $\boldsymbol{x}, \boldsymbol{x}'$ of a finite-dimensional Euclidean space $\mathbb{R}^m$, the Mahalanobis distance is defined as $d ( \boldsymbol{x},\boldsymbol{x}' ) =  \| \boldsymbol{x} - \boldsymbol{x}' \|_{\boldsymbol{A}} \triangleq \sqrt{(\boldsymbol{x} - \boldsymbol{x}')^T \boldsymbol{A} (\boldsymbol{x} - \boldsymbol{x}')}$, where $\boldsymbol{A}=\boldsymbol{A}^T \succeq \boldsymbol{0}$, \ie\ $\boldsymbol{A} \in \mathbb{R}^{m \times m}$ is symmetric positive semi-definite matrix of weights to be determined. Note that when $\boldsymbol{A}$ is not strictly positive definite, it defines a pseudo-metric in $\mathbb{R}^m$. An obvious \ac{DML} approach is to learn this metric in the data's native space, which is tantamount to first linearly transforming the data via a matrix $\boldsymbol{L}$, such that $\boldsymbol{A}=\boldsymbol{L}^T \boldsymbol{L}$, and then measuring distances using the standard Euclidean metric $\left \| \cdot \right \|_2$. 

One possible alternative worth exploring is to search for a non-linear transform prior to measuring Mahalanobis distances, so that performance may improve over the case, where a linear transformation is used. Towards this end, efforts have been recently made to develop kernel-based \ac{DML} approaches. If $\mathcal{X}$ is the original (native) data space, most of these methods choose an appropriate (positive definite) scalar kernel $k: \mathcal{X} \times \mathcal{X} \rightarrow \mathbb{R}$, which gives rise to a \ac{RKHS} $\mathcal{H}$ of functions $f: \mathcal{X} \rightarrow \mathbb{R}$ with inner product $\langle \cdot, \cdot \rangle_{\mathcal{H}}$. This inner product satisfies the (reproducing) property that, for any $x, x' \in \mathcal{X}$, there are functions $\phi_x, \phi_{x'} \in \mathcal{H}$, such that $\langle \phi_x, \phi_{x'}\rangle_{\mathcal{H}} = k ( x, x' )$. The mapping $\phi : x \mapsto \phi_x$ is referred to as the \emph{feature map} and $\mathcal{H}$ is referred to as the (transformed) \emph{feature space} of $\mathcal{X}$, both of which are implied by the chosen kernel. Notice that the feature map may be highly non-linear. Subsequently, these methods learn a metric in the feature space $\mathcal{H}$: $d ( \phi_x, \phi_{x'} ) = \sqrt{\langle (\phi_x - \phi_{x'}), A (\phi_x - \phi_{x'})\rangle_{\mathcal{H}}}$, where $A: \mathcal{H} \rightarrow \mathcal{H}$ is a self-adjoint, bounded, positive-definite operator, preferably, of low rank. Since any element $\phi_x$ of $\mathcal{H}$ may be of infinite dimension, operator $A$ may be described by an infinite number of parameters to be inferred from the data. Obviously, learning $A$ is not feasible by following direct approaches and, therefore, needs to be learned in some indirect fashion. For example, the authors in \cite{Jain2010} pointed out an equivalence between kernel learning and metric learning in the feature space. In specific, they showed that learning $A$ in $\mathcal{H}$ is implicitly achieved by learning a finite-dimensional matrix. 

In this paper, we propose a different \ac{DML} kernelization strategy, according to which a kernel-based, non-linear transform $f$ maps $\mathcal{X}$ into a Euclidean output space $\mathbb{R}^m$, in order to learn a Mahalanobis distance in that output space. This strategy gives rise to two new models that simultaneously learn both the mapping and the output space metric. Leveraged by the Representer Theorem proposed in \cite{Micchelli2005}, all computations of both methods involve only kernel calculations. Unlike previous kernel-based approaches, whose mapping from input to feature space $\mathcal{H}$ cannot be cast into an explicit form, the relevant mappings from input to output space are explicit for both of our methods. Thus, we can access the transformed data in the output space, and this feature can be even used to visualize the data \cite{VanderMaaten2008}, when the output space is $2$- or $3$-dimensional. Furthermore, by specifying the dimensionality of the output space, the rank of the learned metric can be easily controlled to facilitate dimensionality reduction of the original data. 

Our first approach uses an appropriate, but otherwise arbitrary, matrix-valued kernel function and, hence, provides maximum flexibility in specifying the mapping $f$. Furthermore, in this approach, Mahalanobis distances are explicitly parameterized by a weight matrix to be learned. Our second method is similar to the first one, but assumes a specific parameterized matrix-valued kernel function that can be inferred from the data. We show that the Mahalanobis distance is implicitly determined by the kernel function, and thus eliminates the need of learning a weight matrix for the Mahalanobis distances. To demonstrate the merit of our methods, we compare them to standard $k$-NN classification (without \ac{DML}) and other recent kernelized \ac{DML} algorithms, including Large Margin Nearest Neighbor (LMNN) \cite{Weinberger2009}, Information-Theoretic Metric Learning (ITML) \cite{Davis2007} and kernelized LMNN (KLMNN) \cite{Chatpatanasiri2010}. The comparisons are drawn using eight UCI benchmark data sets in terms of recognition performance and show that the novel methods can achieve higher classification accuracy. 
 
\textbf{Related Work} Several previous works have been focused on \ac{DML}. Xing, et. al. \cite{Xing2002} proposed an early \ac{DML} method, which minimizes the distance between similar points, while enlarging the distance between dissimilar points. In \cite{Schultz2004}, relative comparison constraints that involve three points at a time are considered. Neighborhood Components Analysis (NCA) \cite{Goldberger2004} is proposed to learn a Mahalanobis distance for the $k$-NN classifier by maximizing the leave-one-out $k$-NN performance. \cite{Bilenko2004} proposed a \ac{DML} method for clustering. Large Margin Nearest Neighbor (LMNN) \ac{DML} model \cite{Weinberger2009} aims to produce a mapping, so that the $k$-nearest neighbors of any given sample belong to the same class, while samples from different classes are separated by large margins. Similarly, a Support Vector-based method is proposed in \cite{Nguyen2008}. Also, LMNN is further extended to a Multi-Task Learning variation \cite{Parameswaran2010}. Another multi-task \ac{DML} model is proposed in \cite{Zhang2010} that searches for task relationships. In \cite{Huang2009}, the authors proposed a general framework for sparse \ac{DML}, such that several previous works are subsumed. Also, some other \ac{DML} models can be extended to sparse versions by augmenting their formulations. Recently, an eigenvalue optimization framework for \ac{DML} was developed an presented in \cite{Ying2012}. Moreover, the connection between LMNN and Support Vector Machines (SVMs) was discussed in \cite{Do2012}.

Besides the problem of learning a metric in the original feature space, there has been increasing interest in kernelized \ac{DML} methods. In the early work of \cite{Tsang2003}, the Lagrange dual problem of the proposed \ac{DML} formulation is derived, and the \ac{DML} method is kernelized in the dual domain. Information-Theoretic Metric Learning (ITML) \cite{Davis2007} is another kernelized method, which is based on minimizing the Kullback-Leibler divergence between two distributions. The kernelization of LMNN is discussed in \cite{Torresani2007} and \cite{Kulis2009}. Moreover, a Kernel Principal Component Analysis (KPCA)-based kernelized algorithm is developed in \cite{Chatpatanasiri2010}, such that many \ac{DML} methods, such as LMNN, can be kernelized. In \cite{Lu2009}, the Mahalanobis matrix and kernel matrix are learned simultaneously. In \cite{Jain2010} and its extended work \cite{Jain2012}, the authors proposed a framework that builds connections between kernel learning and \ac{DML} in the kernel-induced feature space. Several kernelized models, such as ITML, are covered by this framework. Finally, \ac{MKL}-based metric \ac{DML} is discussed in \cite{Wang2011}.





\section{RKHS for Vector-Valued Functions}
\label{sec:RKHS}
Before introducing our methods, in this section we will briefly review the concept of Reproducing Kernel Hilbert Space (RKHS) for vector-valued functions as presented in \cite{Micchelli2005}. Let $\mathcal{X}$ be an arbitrary set, which we will refer to as \emph{input space}, although it may not actually be a vector space per se. A matrix function $\boldsymbol{K}:\mathcal{X} \times \mathcal{X} \rightarrow  \mathbb{R}^{m \times m}$ is called a \emph{positive-definite matrix-valued kernel}, or simply \emph{matrix kernel}, iff it satisfies the following conditions:

\begin{align}
\label{eq:kernel_condition_symmetry}
& \boldsymbol{K} ( x,x' ) = \boldsymbol{K}^T ( x',x ) \; \; \; &\forall \; x,x' \in \mathcal{X} \\
\label{eq:kernel_condition_pd1}
& \boldsymbol{K}( x,x ) \succeq 0 \; \; \; &\forall \; x \in \mathcal{X} \\
\label{eq:kernel_condition_pd2}
& \bar{\boldsymbol{K}} ( X ) \succeq 0 \; \; \; &\forall \; X \subseteq \mathcal{X} 
\end{align}

\noindent
where $X = \{ x_i \}_{i=1}^{n}$ and $\bar{\boldsymbol{K}} ( X ) \in \mathbb{R}^{mn \times mn}$ is a $n \times n$ block matrix, whose $(i,j)$ block is given as $\bar{\boldsymbol{K}}_{i,j} = \boldsymbol{K}( x_i, x_j ) \in \mathbb{R}^{m \times m}$, where $i,j \in \{1, 2, \ldots, n\}$. According to \cite[Theorem 1]{Micchelli2005}, if $\boldsymbol{K}$ is a matrix kernel, then there exists a unique (up to an isometry) \ac{RKHS} $\mathcal{H}$ of vector-valued functions $f: \mathcal{X} \rightarrow \boldsymbol{R}^m$ equipped with an inner product $ \langle \cdot, \cdot \rangle_{\mathcal{H}}$ that admits $\boldsymbol{K}$ as its reproducing kernel, \ie\, $\forall \; x, x' \in \mathcal{X}$ and $\forall \; \boldsymbol{y}, \boldsymbol{y}' \in \mathbb{R}^m$, there are vector-valued functions $K_{x}\boldsymbol{y}, \; K_{x'}\boldsymbol{y}' \in \mathcal{H}$ that depend on $x, \boldsymbol{y}$ and $x', \boldsymbol{y}'$ respectively, such that it holds

\begin{equation}
	\label{eq:reproducing_property}
	\langle K_{x}\boldsymbol{y}, K_{x'}\boldsymbol{y}' \rangle_{\mathcal{H}} = \boldsymbol{y}^T \boldsymbol{K}( x, x' ) \boldsymbol{y}'
\end{equation}

Note that $K_{x}: \mathbb{R}^m \rightarrow \mathcal{H}$ is a bounded linear operator parameterized by $x \in \mathcal{X}$ and that the function $K_{x}\boldsymbol{y} \in \mathcal{H}$ is such that, when evaluated on $x' \in \mathcal{X}$, it yields

\begin{equation}
	\label{eq:operator_point_evaluation}
	( K_{x}\boldsymbol{y} )(x') = \boldsymbol{K}( x', x ) \boldsymbol{y}
\end{equation}

\section{Fixed Matrix Kernel \ac{DML} Formulation}
\label{sec:Sub_Problem_Formulation}

In this section, we propose our first kernelized \ac{DML} method based on a RKHS for vector-valued functions. Again, let $\mathcal{X}$ be an arbitrary set. Assume we are provided with a training set $\mathcal{T}= \{ ( x_i, \boldsymbol{y}_i ) \}_{i=1,\cdots, n}$, where $x_i \in \mathcal{X}$ and $\boldsymbol{y}_i \in \mathbb{R}^m$, and we are considering the supervised learning task that seeks to infer a distance metric in $\mathbb{R}^m$ along with a mapping $f: \mathcal{X} \mapsto \mathbb{R}^m$ from $\mathcal{T}$. In addition to $\mathcal{T}$, we also assume that we are provided with a real-valued, symmetric \emph{similarity matrix} $\boldsymbol{S} \in \mathbb{R}^{n \times n}$ with entries $s_{i,j} = s(\boldsymbol{y}_i, \boldsymbol{y}_j)$, where $s: \mathbb{R}^m \times \mathbb{R}^m \rightarrow \mathbb{R}_{+}$ is such that $0 \leq s_{i,j} \leq s_{i,i} \; \; \forall \; i,j$. Other than these constraints, the values $s_{i,j}$ can be arbitrary and assigned appropriately with respect to a specific application context. Moreover, let $\boldsymbol{K}\left(x, x'\right)$ be a matrix-valued kernel function (\ie, it satisfies \eref{eq:kernel_condition_symmetry} through \eref{eq:kernel_condition_pd2}) on $\mathcal{X}$ of given form and let $\mathcal{H}$ be its associated \ac{RKHS} of $\mathbb{R}^m$-valued elements. Consider now the following \ac{DML} formulation:

\begin{equation}
\label{eq:formulation_general}
\min_{f, \boldsymbol{L}} \;  \frac{\gamma}{2} \sum_{i, j} s_{i,j}  \| \boldsymbol{L} [ f( x_i ) - f ( x_j ) ] \|_2^2 + \frac{\lambda}{2} \sum_{i} \|  \boldsymbol{L} [ f( x_i ) - \boldsymbol{y}_i ] \|_2^2 + \rho \;\mathrm{tr} ( \boldsymbol{L} ) + \frac{1}{2} \| f  \|_{\mathcal{H}}^2
\end{equation}

\noindent
Notice that $\| \boldsymbol{L} \boldsymbol{y} \|_2 = \| \boldsymbol{y} \|_{\boldsymbol{A}},\; \forall \; \boldsymbol{y} \in \mathbb{R}^m$, where $\boldsymbol{A} \triangleq \boldsymbol{L}^T \boldsymbol{L} \succeq \boldsymbol{0}$. In other words, the Euclidean norms of vector differences appearing in (\ref{eq:formulation_general}) are Mahalanobis distances for the output space. Note that if $\boldsymbol{L}$ is not full-rank, then $\boldsymbol{A}$ is not strictly positive definite, thus $\| \cdot \|_{\boldsymbol{A}}$ will be a pseudo-metric in $\mathbb{R}^m$. The rationale behind this formulation is as follows. The first term, the \emph{collocation} term, forces similar (w.r.t. the similarity measure $s$) input samples to be mapped closely in the output space (unsupervised learning task). The second term, the \emph{regression} term, forces samples to be mapped close to their target values (supervised learning task). In the context of classification tasks, the combination of these two terms aims to force data that belong to the same class to be mapped close to the same cluster. Closeness in the output space is measured via a Mahalanobis metric that is parameterized via $\boldsymbol{L}$. The third term, as we will show later, controls the magnitude of matrix $\boldsymbol{A}$ and facilitates the derivation of our proposed algorithm. Finally, the fourth term is a regularization term and is penalizing the complexity of $f$. Eventually, one can simultaneously learn the output space distance metric and the mapping $f$ through a joint minimization.

The functional of \pref{eq:formulation_general} satisfies the conditions stipulated by the Representer Theorem for Hilbert spaces of vector-valued elements (Theorem 5 in \cite{Micchelli2005}) and, therefore, for a fixed value of $\boldsymbol{L}$, the unique minimizer $\hat{f}$ is of the form:

\begin{equation}
	\label{eq:solution_f}
	\hat{f} = \sum_{i=1}^n  K_{x_i} \boldsymbol{c}_i
\end{equation}

\noindent
where the $m$-dimensional vectors $\{ \boldsymbol{c}_i \}_{i=1}^n$ are to be learned. Notice that, due to \eref{eq:operator_point_evaluation}, the explicit input-to-output mapping is given in \eref{eq:mapping} and
is, in general, non-linear in $x$, if $\mathcal{X}$ is a vector space over the reals.

\begin{equation}
	\label{eq:mapping}
	\hat{f}(x) = \sum_{i=1}^n  \boldsymbol{K} ( x, x_i ) \boldsymbol{c}_i
\end{equation}

\begin{proposition}

\pref{eq:formulation_general} is equivalent to the following minimization problem:

\begin{equation}
\min_{\boldsymbol{c}, \boldsymbol{L}} \; \frac{1}{2} \boldsymbol{c}^T \bar{\boldsymbol{K}} \boldsymbol{c} + \frac{\gamma}{2} \sum_{i, j} s_{ij}\left \| \boldsymbol{L} \boldsymbol{\Gamma}_{ij} \boldsymbol{c} \right \|_2^2 +  \frac{\lambda}{2}\sum_{i} \| \boldsymbol{L} ( \bar{\boldsymbol{K}}_i \boldsymbol{c} - \boldsymbol{y}_i ) \|_2^2 + \rho \;\mathrm{tr} ( \boldsymbol{L} )
\label{eq:formulation_general_matrix_form}
\end{equation}

\noindent 
where $\boldsymbol{c} \triangleq [ \boldsymbol{c}_1^T, \cdots, \boldsymbol{c}_n^T ]^T \in \mathbb{R}^{mn}$, $\bar{\boldsymbol{K}} \in \mathbb{R}^{mn \times mn}$ is the kernel matrix for the training set (as defined for \eref{eq:kernel_condition_pd2}), $\bar{\boldsymbol{K}}_i = \bar{\boldsymbol{K}}(x_i) \triangleq [\boldsymbol{K} ( x_i, x_1 ), \cdots, \boldsymbol{K} ( x_i, x_n ) ] \in \mathbb{R}^{m \times mn}$, and $\boldsymbol{\Gamma}_{ij} = \boldsymbol{\Gamma}(x_i, x_j)  \triangleq \bar{\boldsymbol{K}}_i - \bar{\boldsymbol{K}}_j$. 
\end{proposition}

\noindent
The above proposition can be proved by directly substituting \eref{eq:solution_f} into \pref{eq:formulation_general} and then using \eref{eq:reproducing_property}. Given two samples $x$, $x' \in \mathcal{X}$, the inferred metric will be of the form

\begin{equation}
	\label{eq:metric_fixed_kernel}
	d(x, x') = \| \boldsymbol{L} \boldsymbol{\Gamma}(x, x') \boldsymbol{c} \|_2 = \| \boldsymbol{\Gamma}(x, x') \boldsymbol{c} \|_{\boldsymbol{A}}
\end{equation}

\noindent
with $\boldsymbol{A} = \boldsymbol{L}^T \boldsymbol{L}$. Next, we state a result that facilitates the solution of \pref{eq:formulation_general_matrix_form}.

\begin{proposition}
\label{prop:convex_general_case}
\pref{eq:formulation_general_matrix_form} is convex with respect to each of the two variables $\boldsymbol{c}$ and $\boldsymbol{L}$ individually. 
\end{proposition}

\begin{proof}
The convexity of the objective function, denoted as $Q \left ( \boldsymbol{c}, \boldsymbol{L} \right )$,  with respect to $\boldsymbol{c}$ is guaranteed by the positive semi-definiteness of the corresponding Hessian matrix of $Q$:

\begin{equation}
\frac{\partial^2 Q ( \boldsymbol{c}, \boldsymbol{L} )}{\partial \boldsymbol{c} \partial \boldsymbol{c}^T} =  \bar{\boldsymbol{K}} + \gamma \sum_{i,j} s_{ij} \boldsymbol{\Gamma}_{ij}^T \boldsymbol{L}^T \boldsymbol{L} \boldsymbol{\Gamma}_{ij} + \lambda \sum_i \bar{\boldsymbol{K}}_i^T \boldsymbol{L}^T \boldsymbol{L} \bar{\boldsymbol{K}}_i \succeq \boldsymbol{0}
\label{eq:convex_c_hessian}
\end{equation}

To show the convexity with respect to $\boldsymbol{L}$, we consider each term separately. The convexity of $\| \boldsymbol{L} \boldsymbol{\Gamma}_{ij} \boldsymbol{c} \|_2^2$ stems from the conclusion in \cite[p. 110]{Boyd2004}, which states that $ \| \boldsymbol{Xz} \|_2^2$ is convex with respect to any matrix $\boldsymbol{X}$ for any $\boldsymbol{z}$. For the same reason, $ \| \boldsymbol{L} ( \bar{\boldsymbol{K}}_i \boldsymbol{c} - \boldsymbol{y}_i ) \|_2^2$ is also convex. Finally, $\mathrm{tr} (\boldsymbol{L} )$ is convex in $\boldsymbol{L}$, as shown in \cite[p. 109]{Boyd2004}. Thus, the objective function is also convex with respect to $\boldsymbol{L}$.
\end{proof}

Based on \propref{prop:convex_general_case}, we can perform the joint minimization \pref{eq:formulation_general_matrix_form} by block coordinate descent with respect to $\boldsymbol{c}$ and $\boldsymbol{L}$.  We set the partial derivatives of $Q$ with respect to the two variables to zero and obtain

\begin{equation}
\frac{\partial Q ( \boldsymbol{c}, \boldsymbol{L} )}{\partial \boldsymbol{c}} = \boldsymbol{0}  \Rightarrow \boldsymbol{c} = \lambda ( \frac{\partial^2 Q ( \boldsymbol{c}, \boldsymbol{L} )}{\partial \boldsymbol{c} \partial \boldsymbol{c}^T} )^{\dagger} \sum_{i}\bar{\boldsymbol{K}}_i^T \boldsymbol{L}^T \boldsymbol{L} \boldsymbol{y}_i
\label{eq:solution_c_general} 
\end{equation}

\begin{equation}
\frac{\partial Q ( \boldsymbol{c}, \boldsymbol{L} )}{\partial \boldsymbol{L}} = \boldsymbol{0}  \Rightarrow \boldsymbol{L} =  - \rho ( \gamma \sum_{i,j} s_{ij} \boldsymbol{\Gamma}_{ij} \boldsymbol{c} \boldsymbol{c}^T \boldsymbol{\Gamma}_{ij}^T + \lambda \sum_{i} ( \bar{\boldsymbol{K}}_i \boldsymbol{c} - \boldsymbol{y}_i ) ( \bar{\boldsymbol{K}}_i \boldsymbol{c} - \boldsymbol{y}_i )^T )^{\dagger} 
\label{eq:solution_L_general}
\end{equation}

\noindent 
where $\dagger$ stands for Moore-Penrose pseudo-inversion. One can update $\boldsymbol{c}$ via \eref{eq:solution_c_general} by holding $\boldsymbol{L}$ fixed to its current estimate and then update $\boldsymbol{L}$ via \eref{eq:solution_L_general} by using the most current value of $\boldsymbol{c}$. Repeating these steps until convergence would constitute the basis for the block-coordinate descent to train this model. Due to the calculation of the pseudo-inverse, the time complexity of each iteration, in the worst case scenario, is $O((mn)^3)$. 

As we can observe from \eref{eq:solution_L_general}, since $\boldsymbol{A} = \boldsymbol{L}^T \boldsymbol{L}$, the parameter $\rho$ that appears in the term $\rho \mathrm{tr}(\boldsymbol{L})$ of \pref{eq:formulation_general} directly controls the norm of $\boldsymbol{A}$. Although other regularization terms on $\boldsymbol{L}$ may be utilized in place of $\rho \mathrm{tr} (\boldsymbol{L})$, they may not lead to a simple update equation for $\boldsymbol{L}$, such as the one given in \eref{eq:solution_L_general}. The potential appeal of this formulation stems from the simplicity of the training algorithm combined with the flexibility of choosing a matrix kernel function that is suitable to the application at hand.

\section{Parameterized Matrix Kernel \ac{DML} Formulation}
\label{sec:specialized_kernel}

Our next formulation shares all assumptions with the previous one with the exception that the matrix kernel function $\boldsymbol{K}$ is now parameterized. We shall show that, even though the matrix kernel function is somewhat restricted, it has the property that is able to implicitly determine the output space Mahalanobis metric. To start, we assume a matrix kernel of the form:

\begin{equation}
\label{eq:learnable_kernel}
\boldsymbol{K}( x, x' ) = k(x,x') \boldsymbol{B}
\end{equation}

\noindent
where $k$ is a scalar kernel function that is predetermined by the user and $\boldsymbol{B} \in \mathbb{R}^{m \times m}$ is a symmetric, positive semi-definite matrix, which will be learned from $\mathcal{T}$. Because of this facts, $\boldsymbol{K}$ satisfies \eref{eq:kernel_condition_symmetry} through \eref{eq:kernel_condition_pd2} and, therefore is a legitimate matrix kernel function. The formulation for the alternative \ac{DML} model reads

\begin{equation}
\label{eq:formulation_specialized_B}
\min_{f, \boldsymbol{B}} \; \frac{\gamma}{2} \sum_{i,j} s_{ij} \| f ( \boldsymbol{x}_i ) - f ( \boldsymbol{x}_j ) \|_2^2 + \frac{\lambda}{2} \sum_{i} \| f ( \boldsymbol{x}_i ) - \boldsymbol{y}_i \|_2^2 + \frac{\rho}{2} \| \boldsymbol{B}\|_F^2 + \frac{1}{2} \| f \|_{\mathcal{H}}^2
\end{equation}

\noindent
where $ \| \boldsymbol{B} \|_F^2 \triangleq \mathrm{tr} \{ \boldsymbol{B}^T \boldsymbol{B} \} = \mathrm{tr} \{ \boldsymbol{B}^2 \}$ is the squared Frobenius norm of $\boldsymbol{B}$ and $\mathrm{tr} \{ \cdot \}$ is the matrix trace operator. \pref{eq:formulation_specialized_B} differs from \pref{eq:formulation_general} in a regularization term and in that the former seems to use Euclidean distances in the output space, while the latter uses Mahalanobis distances in the output space with weight matrix $\boldsymbol{A} = \boldsymbol{L}^T \boldsymbol{L}$. As was the case with the formulation of \sref{sec:Sub_Problem_Formulation}, the functional of \pref{eq:formulation_specialized_B} also satisfies the conditions of the Representer Theorem for Hilbert spaces of vector-valued elements and, for fixed value of $\boldsymbol{B}$, the unique minimizer $\hat{f}$ has the same form as the one of \eref{eq:solution_f} and the explicit input-to-output mapping is given as

\begin{equation}
\label{eq:mapping2}
\hat{f}(x) = \sum_{i=1}^n  k(x, x_i) \boldsymbol{B} \boldsymbol{c}_i
\end{equation}

\noindent
which, in all but trivial cases, is again non-linear in $x$, if $\mathcal{X}$ is a vector space over the reals. In a derivation similar to the one found in \sref{sec:Sub_Problem_Formulation}, one can show that \pref{eq:formulation_specialized_B} is equivalent to the following constrained joint minimization problem:

\begin{equation}
\label{eq:formulation_specialized_B_matrix_form}
\min_{\boldsymbol{C}, \boldsymbol{B} \succeq \boldsymbol{0}} \; \frac{\gamma}{2} \mathrm{tr} \{ \boldsymbol{C} \boldsymbol{\widetilde{K}}_{\Delta} \boldsymbol{C}^T \boldsymbol{B}^2 \} + \frac{\lambda}{2} \| \boldsymbol{BC\widetilde{K}} - \boldsymbol{Y} \|_F^2 + \frac{\rho}{2} \| \boldsymbol{B} \|_F^2 +  \frac{1}{2}\mathrm{tr} \{ \boldsymbol{C}^T\boldsymbol{BC} \boldsymbol{\widetilde{K}} \}
\end{equation}

\noindent 
where $\boldsymbol{C} \triangleq [ \boldsymbol{c}_1, \cdots, \boldsymbol{c}_n ] \in \mathbb{R}^{m \times n}$, $\boldsymbol{\widetilde{K}} \in \mathbb{R}^{n \times n}$ is the kernel matrix with $k(x_i, x_j)$ as its $(i,j)$ element, $\boldsymbol{\widetilde{K}}_{\Delta} \triangleq \boldsymbol{\widetilde{K}} [ \mathrm{diag} \{ \boldsymbol{S} \boldsymbol{1}_n \} - \boldsymbol{S} ] \boldsymbol{\widetilde{K}} \in \mathbb{R}^{n \times n}$, where $\mathrm{diag} \{ \cdot \}$ is the operator producing a diagonal matrix with the same diagonal as the operator's argument, $\boldsymbol{1}_n \in \mathbb{R}^n$ is the all-ones vector and $\boldsymbol{Y} \triangleq [ \boldsymbol{y}_1,\cdots,\boldsymbol{y}_n ] \in \mathbb{R}^{m \times n}$. The learned metric will be of the form

\begin{equation}
\label{eq:metric_parameterized_kernel}
d(x, x')  = \| \boldsymbol{B} \boldsymbol{C} [ \boldsymbol{\widetilde{k}}(x) - \boldsymbol{\widetilde{k}}(x') ] \|_2 = \| \boldsymbol{C} [ \boldsymbol{\widetilde{k}}(x) - \boldsymbol{\widetilde{k}}(x') ] \|_{\boldsymbol{A}} 
\end{equation}

\noindent
where $\boldsymbol{\widetilde{k}}(x) \triangleq [ k(x,x_1), \ldots k(x,x_n) ]^T$ and, in this case, $\boldsymbol{A} = \boldsymbol{B}^2$. It is readily seen that the matrix $\boldsymbol{B}$ specifying the matrix kernel function also determines the Mahalanobis distance in the output space $\boldsymbol{R}^m$. Therefore, this model implicitly learns the Mahalanobis distance by learning the $\boldsymbol{B}$ matrix in the kernel function.

\begin{proposition}
\pref{eq:formulation_specialized_B_matrix_form} is convex with respect to each of the two variables $\boldsymbol{C}$ and $\boldsymbol{B}$.
\label{prop:convex_specific_kernel}
\end{proposition}

\begin{proof}
The proof is based on the following facts outlined in \cite[sec. 3.6]{Boyd2004}: (a) A matrix-valued function $g$ is \emph{matrix convex} if and if for any $\boldsymbol{z}$, $\boldsymbol{z}^T g \boldsymbol{z}$ is convex. (b) Suppose a matrix-valued function $g$ is matrix convex and a real-valued function $h$ is convex and non-decreasing. Then, $h \circ g$ is convex, where $\circ$ denotes function composition. (c) The function $\mathrm{tr} \{ \boldsymbol{W} \boldsymbol{X} \}$ is convex and non-decreasing in $\boldsymbol{X}$, if $\boldsymbol{W} \succeq \boldsymbol{0}$. In what follows, we show convexity for each term in \pref{eq:formulation_specialized_B_matrix_form}. Since $\mathrm{tr} \{ \boldsymbol{C}^T\boldsymbol{BC} \boldsymbol{\widetilde{K}} \} = \mathrm{tr} \{  \boldsymbol{C\widetilde{K}}\boldsymbol{C}^T \boldsymbol{B} \}$ and $\boldsymbol{C\widetilde{K}}\boldsymbol{C}^T \succeq \boldsymbol{0}$, therefore $\mathrm{tr} \{  \boldsymbol{C\widetilde{K}}\boldsymbol{C}^T \boldsymbol{B} \}$ is convex with respect to $\boldsymbol{B}$ based on facts (b) and (c). To show the convexity with respect to $\boldsymbol{C}$, note that the matrix-valued function $g ( \boldsymbol{C} ) = \boldsymbol{C}^T\boldsymbol{BC}$ is matrix convex with respect to $\boldsymbol{C}$ based on $\boldsymbol{B} \succeq \boldsymbol{0}$ and fact (a). Thus, with $\boldsymbol{\widetilde{K}} \succeq \boldsymbol{0}$ and fact (b) and (c), we achieve the convexity. The same method is employed to prove the convexity of the other three terms (note that $\widetilde{K}_{\Delta} \succeq \boldsymbol{0}$).
\end{proof}

Based on \propref{prop:convex_specific_kernel}, we can again apply a block coordinate descent algorithm to solve \pref{eq:formulation_specialized_B_matrix_form}. If $\tilde{Q} ( \boldsymbol{C}, \boldsymbol{B} )$ is the relevant objective function, we set the partial derivative of $\tilde{Q} ( \boldsymbol{C}, \boldsymbol{B} )$ with respect to $\boldsymbol{C}$ zero and obtain:

\begin{equation}
\frac{\partial \tilde{Q} ( \boldsymbol{C}, \boldsymbol{B} )}{\partial \boldsymbol{C}} = \boldsymbol{0} \Rightarrow \boldsymbol{C} + \gamma \boldsymbol{BC} \boldsymbol{\widetilde{K}}_{\Delta} \boldsymbol{\widetilde{K}}^{-1} + \lambda \boldsymbol{BC\widetilde{K}} = \lambda \boldsymbol{Y}
\label{eq:specialized_form_solve_C_1}
\end{equation}

As noted in \cite{Lancaster1970}, this matrix equation can be solved for $\boldsymbol{C}$ as follows:

\begin{equation}
\mathrm{vec} ( \boldsymbol{C} ) = \lambda ( \boldsymbol{I} + \gamma ( \boldsymbol{\widetilde{K}}_{\Delta}\boldsymbol{\widetilde{K}}^{-1} ) \otimes \boldsymbol{B} + \lambda \boldsymbol{\widetilde{K}} \otimes \boldsymbol{B} )^{-1} \mathrm{vec} ( \boldsymbol{Y} )
\label{eq:specialized_form_solve_C_2}
\end{equation}

To find the optimum $\boldsymbol{B}$ for fixed $\boldsymbol{C}$, due to the constraint $\boldsymbol{B} \succeq \boldsymbol{0}$, we use a projected gradient descent method. In each iteration, we update $\boldsymbol{B}$ using the traditional gradient descent rule: $\boldsymbol{B} \leftarrow \boldsymbol{B} - \alpha \bigtriangledown_{\boldsymbol{B}} \tilde{Q} ( \boldsymbol{C},\boldsymbol{B} )$, where $\alpha > 0$ is the step length, followed by projecting the updated $\boldsymbol{B}$ onto the cone of positive semi-definite matrices. Since $\tilde{Q} ( \boldsymbol{C},\boldsymbol{B} )$ is convex with respect to $\boldsymbol{B}$ for fixed $\boldsymbol{C}$, this procedure is able to find the optimum solution for $\boldsymbol{B}$. The gradient with respect to $\boldsymbol{B}$ is given as

\begin{equation}
	\label{eq:specialized_form_solve_B_2}
	\frac{\partial \tilde{Q}\ ( \boldsymbol{C},\boldsymbol{B} )}{\partial \boldsymbol{B}} = \boldsymbol{G} + \boldsymbol{G}^T - \boldsymbol{G}\odot \boldsymbol{I}  
\end{equation}

\noindent
where $\odot$ is the Hadamard matrix product and $\boldsymbol{G}$ is defined as

\begin{equation}
	\label{eq:Gexpression}
	\boldsymbol{G} \triangleq \boldsymbol{B} [  \boldsymbol{C} ( \gamma \widetilde{\boldsymbol{K}}_\Delta + \lambda \boldsymbol{K}^2 ) \boldsymbol{C}^T + \rho \boldsymbol{I} ] - ( \lambda \boldsymbol{Y} - \frac{1}{2} \boldsymbol{C} ) \boldsymbol{K} \boldsymbol{C}^T
\end{equation}

Therefore, for each iteration, the time complexity of updating $\boldsymbol{C}$ is $O((mn)^3)$, due to the calculation of a matrix inverse. When updating $\boldsymbol{B}$, the time complexity is determined by the convergence speed of the projected gradient descent method.

\section{Experiments}
\label{sec:Experiments}

In this section, we evaluate the performance of our two kernelized \ac{DML} methods on classification problems. Towards this purpose, we opt to set $\boldsymbol{y}_i = \boldsymbol{y}^{k(i)},\; \forall i \in \{ 1, 2, \ldots, n \}$, where $k(i) \in \{ 1, 2, \ldots, c \}$ is the class label of the $i^{th}$ sample and $\boldsymbol{y}^k$ is an appropriately chosen prototype target vector for the $k^{th}$ class. Additionally, we choose to evaluate the pair-wise sample similarities as $s_{i,j} = [ \boldsymbol{y}_i = \boldsymbol{y}_j ]$, where $[ predicate ]$ denotes the result of the Iversonian bracket, \ie\ it equals $1$, if $predicate$ evaluates to true, and $0$, if otherwise. After training each of these models, we employ a \ac{KNN} classifier to label samples in the range space of $f$; the classifier uses the models' learned metrics (given by \eref{eq:metric_fixed_kernel} and \eref{eq:metric_parameterized_kernel}) to establish nearest neighbors. 

We compare our methods with several other approaches. The first one labels samples of the original feature space via the $k$-NN classification rule using Euclidean distances and, provides a baseline for the accuracy that can be achieved for each classification problem we considered. The second one relies on a popular \ac{DML} method, namely the Large Margin Nearest Neighbor (LMNN) \ac{DML} method \cite{Weinberger2009}. We also selected two kernelized approaches for comparison, namely, Information-Theoretic Metric Learning (ITML) \cite{Davis2007} and kernelized LMNN (KLMNN) \cite{Chatpatanasiri2010}.

We evaluated all approaches on eight datasets from the UCI repository, namely, White Wine Quality (\textit{Wine}), Wall-Following Robot Navigation (\textit{Robot}), Statlog Vehicle Silhouettes (\textit{Vehicle}), Molecular Biology Splice-junction Gene Sequences (\textit{Molecular}), Waveform Database Generator Version 1 (\textit{Wave}), Ionosphere (\textit{Iono}), Cardiotocography (\textit{Cardio}), Pima Indians Diabetes (\textit{Pima}). For all datasets, each class was equally represented in number of samples. An exception is the original \textit{Wine} dataset that has eleven classes, eight of which are poorly represented; for this dataset we only chose data from the other three classes.


For our model with general matrix kernel function $\boldsymbol{K}$,  we chose the diagonal matrix $\boldsymbol{K} ( \boldsymbol{x}_i, \boldsymbol{x}_j ) = \mathrm{diag}  \{ [ k_1 ( \boldsymbol{x}_i, \boldsymbol{x}_j ), \ldots, k_m ( \boldsymbol{x}_i, \boldsymbol{x}_j ) ]^T \}$, where $k_1$ through $k_m$ were Gaussian kernel functions with different spreads. For the second model, where $\boldsymbol{K} = k \cdot \boldsymbol{B}$, we also chose $k$ to be a Gaussian kernel. During the test phase for all experiments, the parameters $\gamma$, $\lambda$, $\rho$, the output dimension $m$, the Gaussian kernel's spread parameter $\sigma$ and the number of nearest neighbors $\kappa$ to be used by the \ac{KNN} classifier are selected through cross-validation. Training of the models was performed using $10$\% and $50$\% of each data set. In the sequel, we provide the experimental results in figures, which display the average classification accuracies over $20$ runs. Also, the error bars correspond to a $95$\% confidence interval of the estimated accuracies.

\begin{figure}[ht]
\centering
	\subfloat[]{
		\includegraphics[width=8cm]{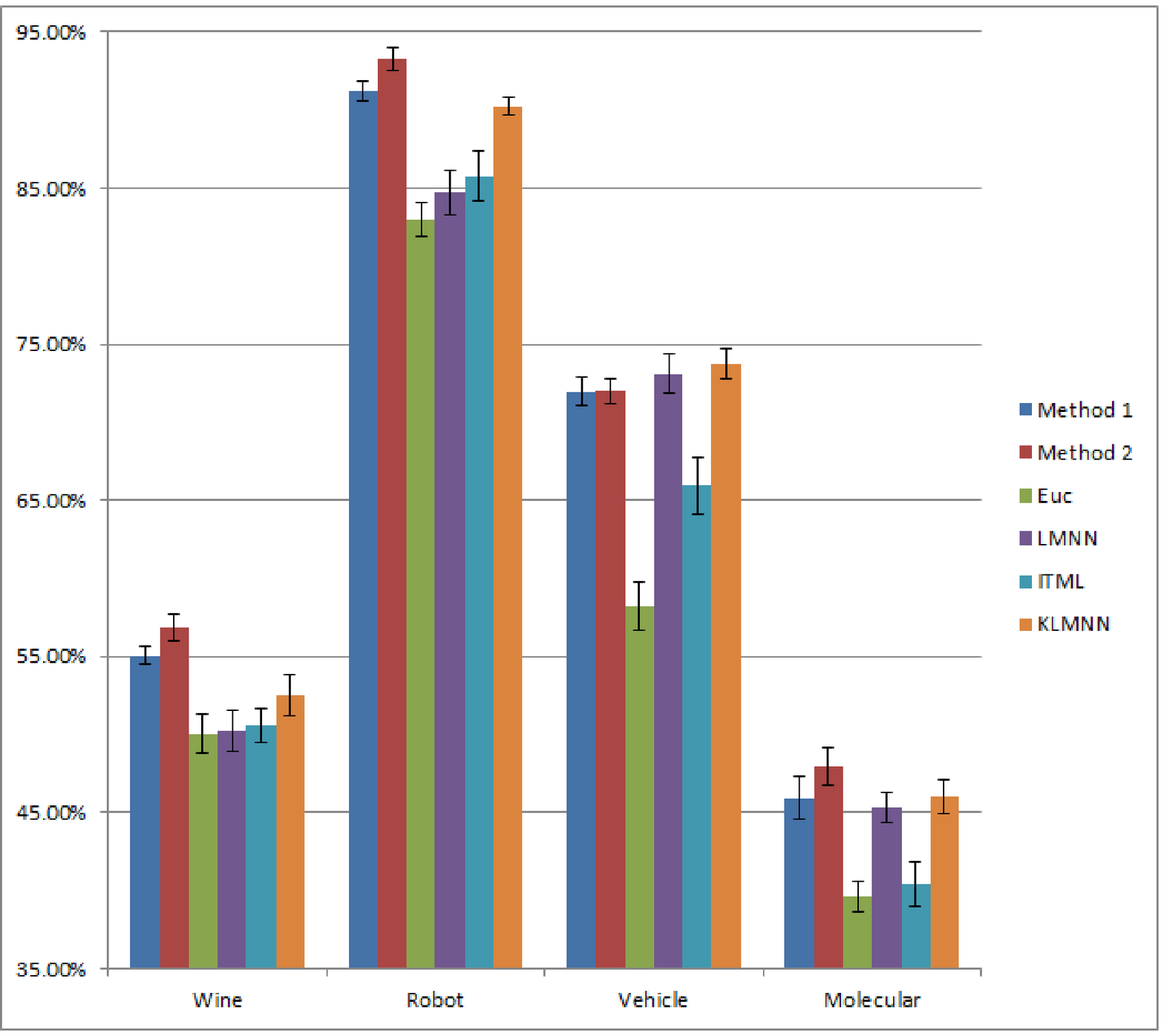}
		\label{fig:first_4_10_percent}}
	\subfloat[]{
		\includegraphics[width=8cm]{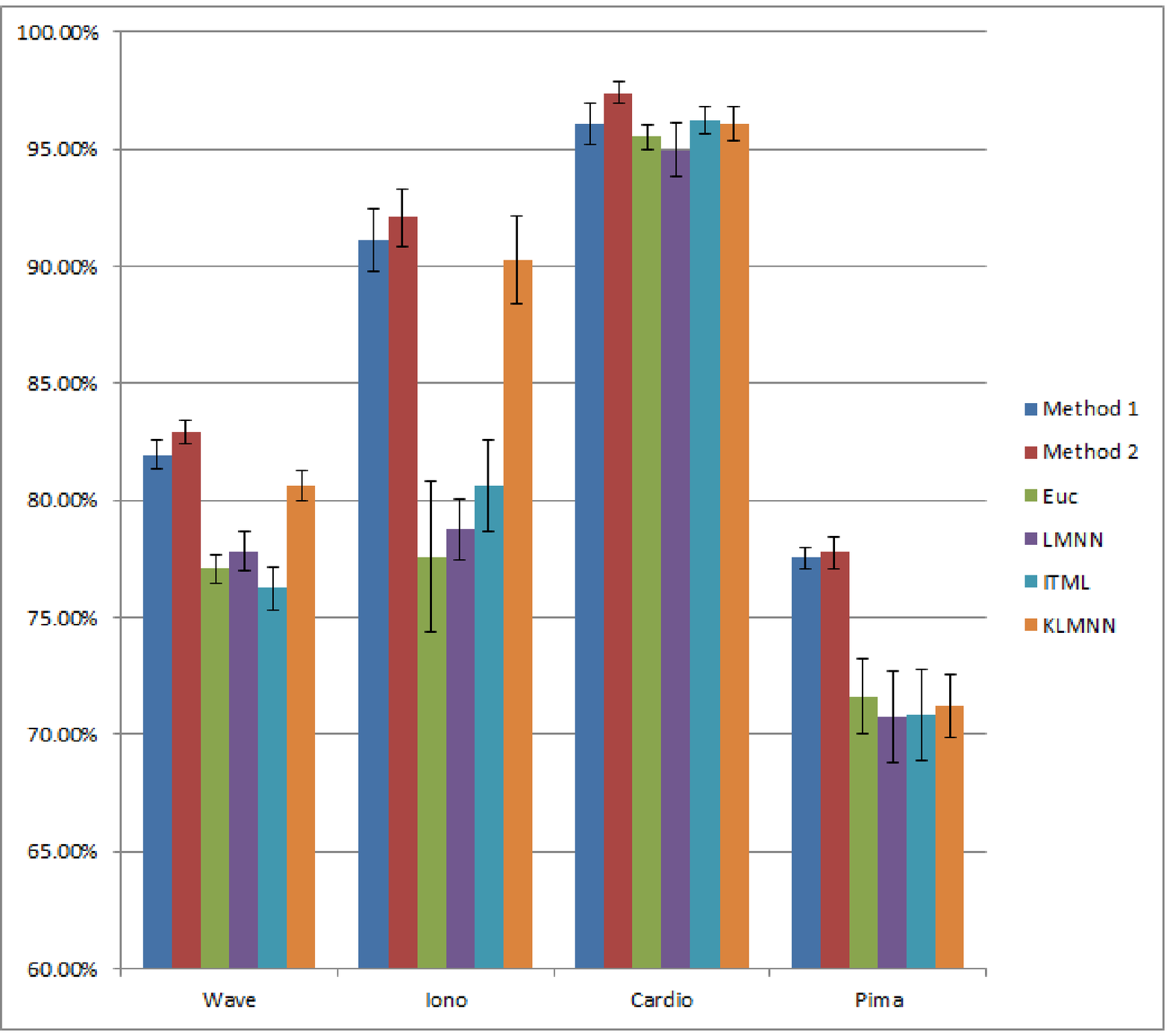}
		\label{fig:last_4_10_percent}}
\caption{Experimental results for $10$\% training data. Average classification performance over $20$ runs for each data set and each method is shown. Error bars indicate $95$\% confidence intervals.}
\label{fig:ten_percent}
\end{figure}

We first discuss the results in the case where we used only 10\% of the training data; they are depicted in \fref{fig:ten_percent}. Our first model with general kernel function $\boldsymbol{K}$ is named as ``Method 1'', and the second model with specified kernel function $\boldsymbol{K} = k \cdot \boldsymbol{B}$ is called ``Method 2''. For almost all datasets, we observe that all five \ac{DML} methods outperform the scheme involving no transformation of the original feature space (\ie, the output space coincided with the original feature space) and labeling samples via Euclidean-distance \ac{KNN} classification. This remarkable fact underlines the potential benefits of \ac{DML} methods. Moreover, we observe that kernelized methods usually outperform LMNN. This observation may partly justify the use of a nonlinear mapping for \ac{DML}. Furthermore, we observe from the figure that both of our methods typically outperform the other four approaches. More specifically, the proposed two models achieve the highest accuracy across all datasets with the only exception on the \textit{Vehicle} dataset, where ITML and KLMNN outperform slightly. It is worth mentioning that, for the \emph{Pima} data set, none of the other three \ac{DML} methods can enhance the performance compared to the baseline \ac{KNN} classification, while our methods achieve significant improvements.

\begin{figure}[ht]
\centering
	\subfloat[]{
		\includegraphics[width=8cm]{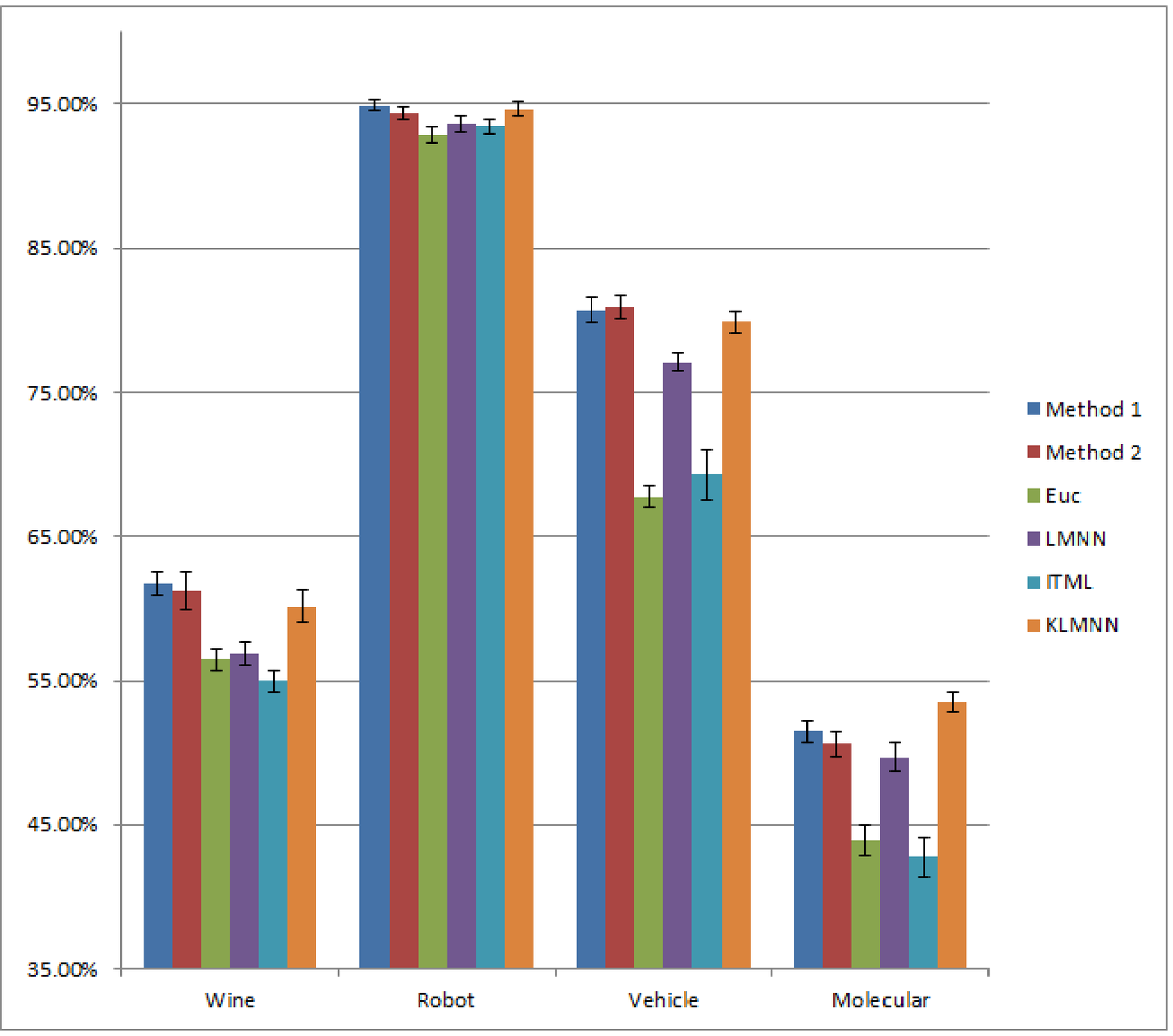}
		\label{fig:first_4_50_percent}}
	\subfloat[]{
		\includegraphics[width=8cm]{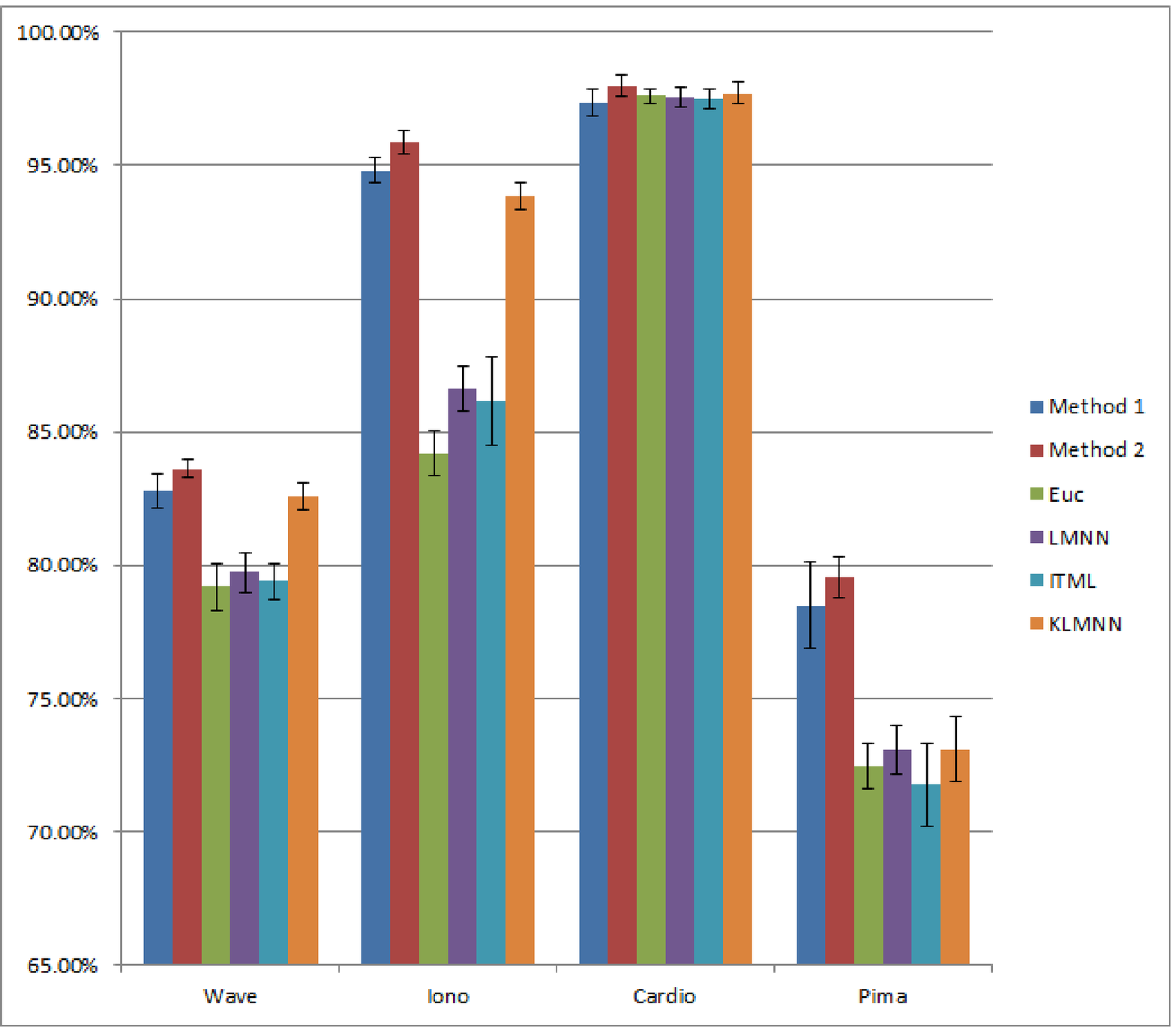}
		\label{fig:last_4_50_percent}}
\caption{Experimental results for $50$\% training data. Average classification performance over $20$ runs for each data set and each method is shown. Error bars indicate $95$\% confidence intervals.}
\label{fig:fifty_percent}
\end{figure}

Similar conclusions can be drawn regarding the results generated by using $50$\% of the training data. These results are depicted in \fref{fig:fifty_percent}. Our methods outperform all the other four methods for most datasets. An exception occurs for the \textit{Molecular} dataset, where KLMNN achieves higher performance than ours. In the case of \textit{Robot} and \textit{Cardio} datasets, all methods perform similarly well. The reason might be that, with enough data, all of the models can be trained well enough to achieve close to optimal performance. For the \textit{Pima} data set, again, our methods achieve much better results than all other four methods. It is also important to note that, for our Method 1, despite the relatively simple form of the matrix kernel function we opted for, the resulting model demonstrated very competitive classification accuracy across all datasets. One would likely expect even better performance, if a more sophisticated matrix kernel function is used.

For the sake of visualizing the distribution of the transformed \emph{Robot} data via our models in $2$ dimensions, we provide \fref{fig:robot_KPCA} and \fref{fig:robot_visualize}. Similar to \cite{Jain2010}, we compare the produced mappings of our methods to Kernel Principal Component Analysis (KPCA). KPCA's $2$-dimensional principal subspace was identified based on $10$\% of the available training data, \ie, $100$ training patterns, and the test points were projected onto that subspace. The same training samples were also used for training our two models, which used a Gaussian kernel function and a spread parameter value $\sigma$ that maximized \ac{KNN}'s classification accuracy.

\begin{figure}[ht]
\begin{center}
		\centering
		\includegraphics[width=3.2in]{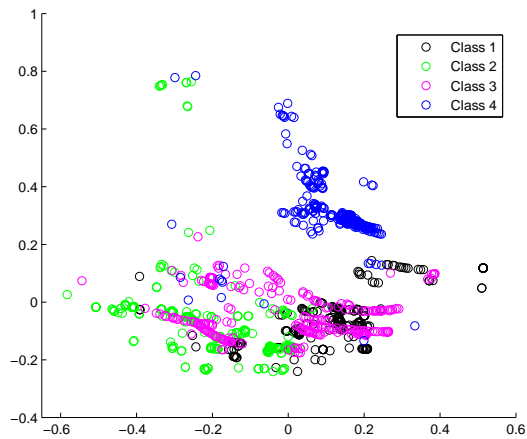}
		\caption{Visualization of the \textit{Robot} data set by applying KPCA.}
		\label{fig:robot_KPCA}
\end{center}
\end{figure}

\begin{figure}[ht]
\centering
	\subfloat[Method 1]{
		\includegraphics[width=8cm]{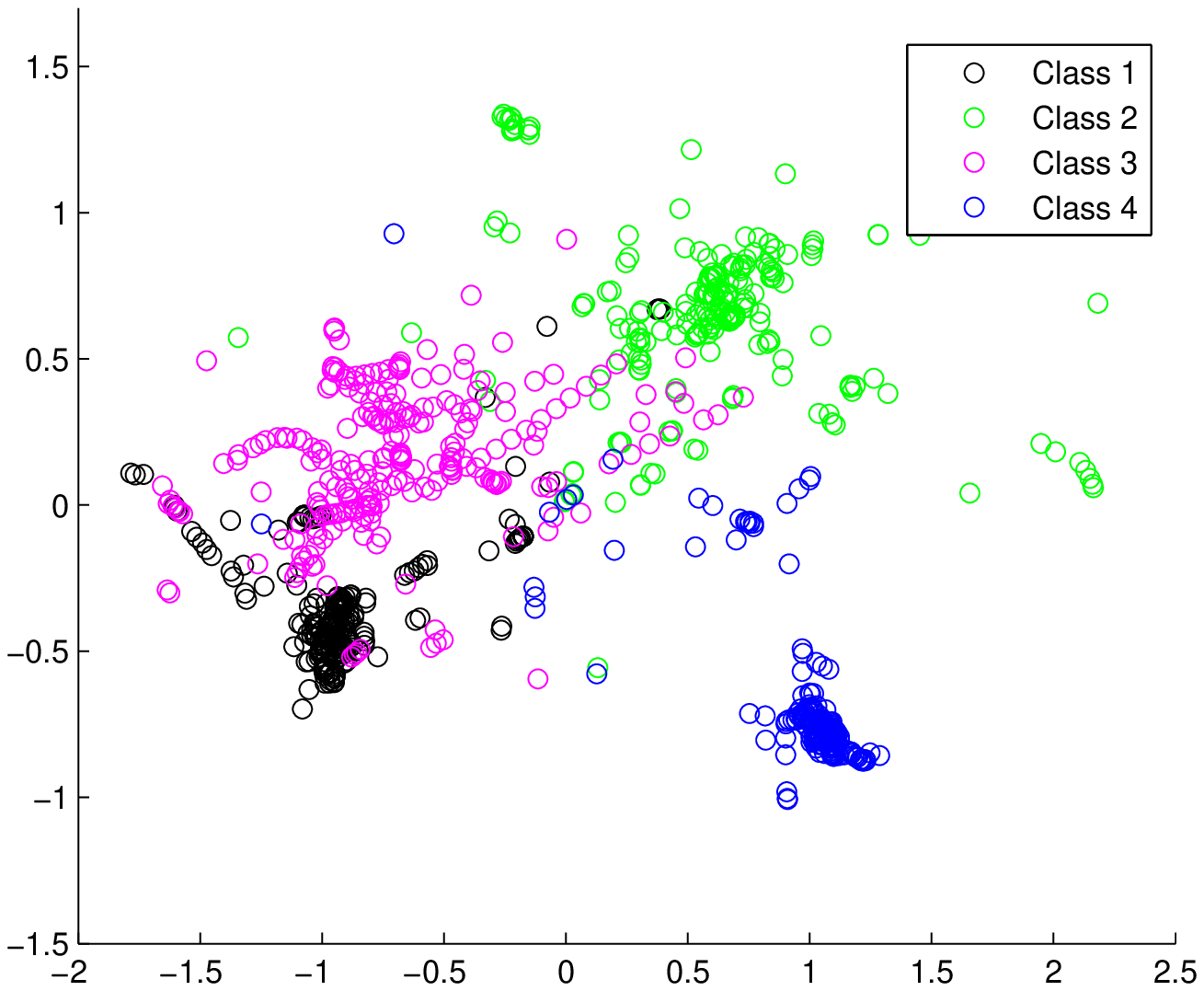}
		\label{fig:robot_general}}
	\subfloat[Method 2]{
		\includegraphics[width=8cm]{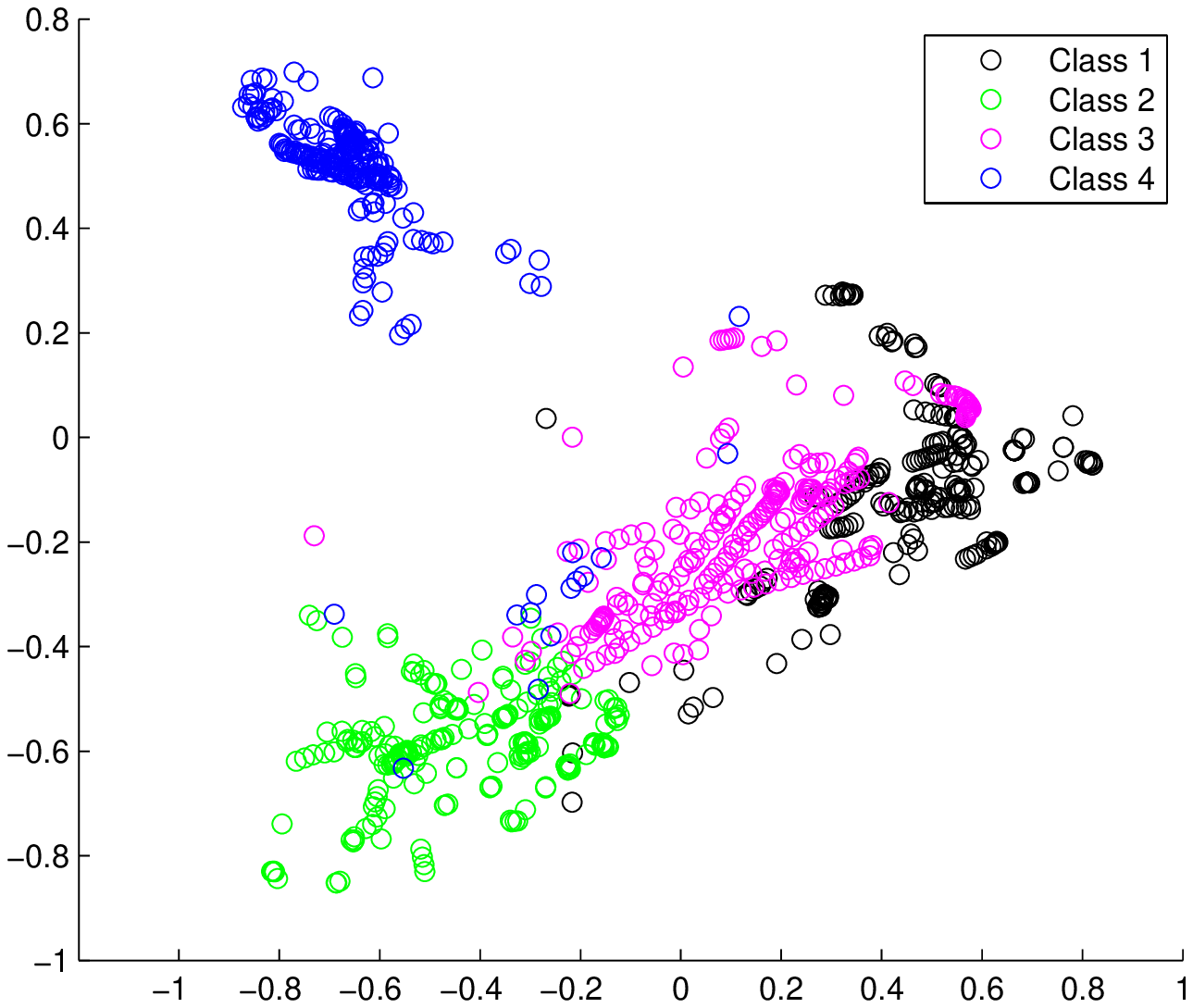}
		\label{fig:robot_B}}
\caption{Visualization of the \textit{Robot} data set by applying our methods.}
\label{fig:robot_visualize}
\end{figure}

From \fref{fig:robot_KPCA} we observe that KPCA's projection may only promote good discrimination between samples drawn from class $4$ versus the rest. On the other hand, in \fref{fig:robot_general} and \fref{fig:robot_B}, all four classes are reasonably well-clustered in the output space obtained by our two methods. This may explain why our methods are able to achieve high classification accuracy, even when only $10$\% of the available data are used for training.

\section{Conclusions}
\label{sec:Conclusions}

\acresetall

In this paper, we proposed two new kernel-based \ac{DML} methods, which rely on \acp{RKHS} of vector-valued functions. Via a mapping $f$, the two methods map data from their original space to an output space, whose dimension can be directly controlled. Subsequent distance measurements are performed in the output space via a Mahalanobis metric. The first proposed model uses a general matrix kernel function and, thus, provides significant flexibility in modeling the input-to-output space mapping. On the other hand, the second proposed method uses a more restricted matrix kernel function, but has the advantage of implicitly determining the Mahalanobis metric. Furthermore, its matrix kernel function can be learned from data. Unlike previous kernel-based approaches, the relevant $f$ mappings are explicit for both of our two methods. Combined with the fact that the output space dimensionality can be directly specified, the models can also be used for dimensionality reduction purposes, such as for visualizing the data in $2$ or $3$ dimensions. Experimental results on eight UCI benchmark data sets show that both of the proposed methods can achieve higher performance in comparison to other traditional and kernel-based \ac{DML} techniques.

\section*{Acknowledgements}

 C. Li acknowledges partial support from \ac{NSF} grant No. 0806931. Also, M. Georgiopoulos acknowledges partial support from \ac{NSF} grants No. 0525429, No. 0963146, No. 1200566 and No. 1161228. Any opinions, findings, and conclusions or recommendations expressed in this material are those of the authors and do not necessarily reflect the views of the \ac{NSF}.

\bibliographystyle{plainurl}
\bibliography{ArXiv2013paperA}

\end{document}